%% file: paper.tex
\documentclass[preprint,3p,times]{elsarticle}
\usepackage{times}
\usepackage{helvet}
\usepackage{courier}
\usepackage{url}
\usepackage{graphicx}

\journal{Journal}
\usepackage[T1]{fontenc} %
\usepackage[utf8]{inputenc}

\usepackage{lipsum}

\usepackage{amsmath}
\usepackage{amssymb}

\usepackage{algorithm}
\usepackage{algpseudocode}

\usepackage{bm}

\usepackage{booktabs}
\newcommand{\ra}[1]{\renewcommand{\arraystretch}{#1}}
\ra{1.2}
\input{local-commands.tex}

\begin{document}
 \begin{frontmatter}

		 \title{A Correctness Result for Synthesizing Plans With  Loops  in  Stochastic Domains}
	 
	 		\cortext[cor]{Corresponding author.}

	 		 \address[edinburgh]{School of Informatics, University of Edinburgh,
	 		Edinburgh, UK.}
	 		 \address[turing]{Alan Turing Institute,
	 		London, UK.}

	 \author[edinburgh]{Laszlo Treszkai}\ead{laszlo.treszkai@gmail.com}
	  		\author[edinburgh,turing]{Vaishak Belle\corref{cor}}
	 		\ead{vaishak@ed.ac.uk}

\begin{abstract}
Finite-state controllers (FSCs), such as plans with loops, are powerful and compact representations of action selection widely used in robotics, video games and logistics. There has been steady progress on synthesizing FSCs in deterministic environments, but the algorithmic machinery needed for lifting such techniques to stochastic environments is not yet fully understood.
While the derivation of FSCs has received some attention in the context of discounted expected reward measures, they are often solved approximately and/or without correctness guarantees. In essence, that makes it difficult to analyze fundamental concerns such as: do all paths terminate, and do the majority of paths reach a goal state?

In this paper, we present new theoretical results on a generic technique for synthesizing FSCs in stochastic environments, allowing for highly granular specifications on termination and goal satisfaction.
\end{abstract}

\begin{keyword}
	Plan and program synthesis \sep
	Stochastic domains \sep
	Loops in plans and programs \sep
	Stochastic algorithms \sep 
	Planning in robotics 
\end{keyword}

\end{frontmatter}

\section{Introduction} %
\label{sec:introduction}

Finite-state controllers (FSCs), such as plans with loops, are powerful and compact representations of action selection widely used in robotics, video games and logistics. In AI, FSCs are much sought after for automated planning paradigms such as \emph{generalized planning}, as in Figure~\ref{fig:hall-a-one}, where one attempts to synthesize a controller that works in multiple initial states. 
Such controllers are usually hand-written by domain experts, which is problematic when expert knowledge is either unavailable  or unreliable. 
To that end, the automated synthesis of FSCs has received considerable attention  in recent years, e.g., \cite{levesque2005-kplanner,Bonet2009-Automatic-derivation,Srivastava2010-Foundations-thesis,Pralet2010-constraint-based,Hu2013-and-or,Srivastava2015-Tractability}. Of course, FSCs synthesis is closely related to program synthesis \cite{levesque2005-kplanner}, and FSCs are frequently seen as program-like plans \cite{lin1998robots}, and recent synthesis literature involves an exciting exchange of technical insights between the two fields \cite{Srivastava2010-Foundations-thesis}; representative examples include the use of program synthesis to infer high-level action types \cite{schmid2000applying}, and the use of partial order planning for imperative program synthesis \cite{ireland2006combining}. 

Naturally, from an algorithmic perspective, the two most immediate questions are: in which sense are controllers \emph{correct}, and how do we \emph{synthesize} controllers that are \emph{provably correct}? In classical deterministic settings, plan paths can only be extended uniquely, so it suffices to show that there is a terminating path that reaches the goal state. Ideally, then, what we seek is a procedure that is both \emph{sound} (i.e., all synthesized controllers are correct) and \emph{complete} (i.e., if there is a plan, then the procedure finds it).

\begin{figure}[t]\centering 
    \tikz{\input{tikz-hall-a-one.tex}}
	\label{fig:hall-a-one}
	\caption{\emph{Above}: A planning problem where the agent is initially in cell A, and the goal is to visit cell B and go back to cell A.
	At each state, possible observations are $A$, $B$, or --.
	\emph{Below}: A correct finite state controller for this problem.
	The circles are controller states, and an edge $\smash{q\xrightarrow{\raisebox{-0.5ex}[1ex][0pt]{$\scriptstyle o:a$}}q'}$ means
	``do $a$ when the observation is $o$ in controller state~$q$, and then switch to controller state~$q'$.'' The reader may observe that this controller works for any number of states between A and B in the domain, e.g., for $(A, -, B)$ as well as $(A, -, - -, -, B)$.}
\end{figure}
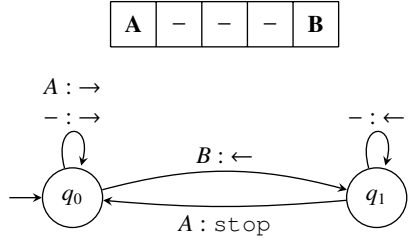

The compact nature of FSCs makes them particularly attractive for mobile robots \cite{mataric2007robotics}, among other domains where there is inherent stochasticity and actions are noisy. To a first approximation, in the presence of non-probabilistic nondeterminism, it is common practice to make meta-level assumptions, such as disallowing repeated {configurations} of (state, action) pairs. However, in stochastic environments, that is almost always an unreasonable assumption. Consider a robot attempting to grip an object: the first and second attempt may fail, but perhaps the third succeeds. Structurally, the first two failures are identical: while a domain expert might find a way to distinguish the two states for the planner, from a robustness viewpoint, of course, it is more desirable when algorithms operate  without such meta-level assumptions. In this regard,  the algorithmic machinery needed for lifting FSC synthesis techniques to stochastic environments is not yet fully understood.

 More generally, we identify the following desiderata: \begin{itemize}
     \item[\textbf{D1.}]  The planner should cope with plan paths that do not terminate in a goal state;
     \item[\textbf{D2.}] The planner should correctly account for how a looping history affects goal probabilities, distinguishing loops that never terminate and loops that can be extended into goal histories; and
     \item[\textbf{D3.}] The planner should recognize when a combination of loops never terminates, even if the loops by themselves appear to be possibly terminating.
 \end{itemize}{}
Implicit in these desiderata is the idea that the planner should leverage the likelihood of action outcomes,  because \emph{(a)}~these likelihoods are informative about which action outcomes are more likely than others, and \emph{(b)}~in the presence of repeating configurations, probabilities allow for a natural tapering of the likelihood of paths.

In this paper,  we present new theoretical results  on a generic technique for synthesizing FSCs in stochastic environments, by means of a probabilistic extension of  AND-OR search. We provide a careful  analysis of how to maintain upper and lower bounds of the likelihoods of paths, so that one can naturally deal with tapering probabilities, arising from repeating configurations. In particular, it allows us to plan for highly granular specifications, such as: generate a FSC under the requirement that >80\% of the paths terminate, and >60\% of the paths reach the goal state. Most significantly, we prove that our algorithm is both \emph{sound} and \emph{complete}.

\input{sec-formalisation.tex}

\input{sec-HG.tex}

\input{sec-problems.tex}

\input{sec-pandor.tex}

\section{Related Work}

Our results are related to a number of recent approaches on bounded search and FSC synthesis, but as we discuss below, the nature of our results and the thrust of our proof strategy is significantly different from these approaches. At the outset,  our contributions should be seen as a full generalization of \cite{Hu2013-and-or} to stochastic domains, in that it provides a generic technique for FSC synthesis in a whole range of planning frameworks (cf.  \cite{Hu2013-and-or}) that can now be considered with probabilistic nondeterminism. 

The work of \cite{Hu2013-and-or} is positioned in the area of generalized planning. We will briefly touch on approaches to generating loopy plans, and then discuss related correctness concerns.

Early approaches to loopy plans can be seen as deductive methodologies, often  influenced by program synthesis and correctness  \cite{DBLP:journals/cacm/Hoare69}.  Manna and Waldinger \cite{DBLP:journals/toplas/MannaW80}  obtained recursive plans by matching induction rules, and  \cite{DBLP:conf/aips/StephanB96}  refine generic plan specifications, but required  input from humans. See  \cite{DBLP:conf/kr/MagnussonD08} for a recent approach using  induction.

Most recent proposals differ considerably from this early work using deduction: \begin{itemize}
	\item \cite{levesque2005-kplanner}  expects two parameters with the planning problem; the approach plans for the first parameter, winds it to form loops and tests it for the second. 
	\item \cite{winner07_loop} synthesize a plan sequence with partial orderings, and exploit repeated occurrences of subplans to obtain loops. 
	\item \cite{Srivastava2010-Foundations-thesis} considers an abstract state representation that groups objects into equivalences classes, the idea being that any concrete plan can be abstracted wrt these classes and repeated occurrences of subplans can be leveraged to generate compact loopy plans.  
	\item \cite{Bonet2009-Automatic-derivation}  integrate the dynamics of a memoryless plan with a planning problem, and convert that to a conformant planning problem; the solution to this latter problem is shown to generalize to multiple instances of the original problem. 
	\item \cite{Hu2013-and-or} propose a bounded AND/OR search procedure that is able to synthesize loopy plans, which is what we build on. 
\end{itemize}

On the matter of correctness,  \cite{DBLP:conf/aaai/Levesque96}  argued that generalized plans be tested for termination and correctness against all problem instances; \cite{lin1998robots}  extended this account to define goal achievability. In later work, \cite{Cimatti2003-model-checking} defined the notions of weak, strong and strong cyclic solutions in the presence of nondeterminism.\footnote{Variant additional stipulations in the literature include things like     \emph{fairness}, where every outcome of a nondeterministic action must occur infinitely often \cite{DBLP:conf/ijcai/BonetG15}.} These notions are widely used in the planning community \cite{DBLP:conf/ijcai/BonetG15}; see, for example,  \cite{Bertoli2006337} for an account of strong planning with a sensor model. Recently,   \cite{Srivastava2015-Tractability}  synthesize loopy plans in domains with nondeterministic quantitative effects,  for which strong cyclic solutions are studied.  Our account of correctness is based on \cite{belle2016-generalized-planning}, which generalized Levesque's account \cite{DBLP:conf/aaai/Levesque96}.

Synthesizing FSCs is a very active area of research within Markov Decision Processes (MDPs) and partially observable MDPs (POMDPs) \cite{Meuleau1999-pomdps-finite,Poupart2003-Bounded,Amato2007-quadratically-constrained}.
But the majority of algorithms in this space either solve  an approximation of the problem, or they come without correctness guarantees. 
In contrast, emphasizing correctness, \cite{junges2018finite} show how FSC synthesis for POMDPs can be reduced to parameterized Markov chains, {under the requirement of \textit{almost-sure plans that do not enter bad states.}} Similarly, \cite{chatterjee2016symbolic} propose the synthesis of \textit{almost-sure plans} by means of a SAT-based oracle. Not only are the algorithms significantly different from our own, but the correctness specification too is formulated differently. Thus, our contributions are complementary to this major body of work, and orthogonal to a large extent; for the future, it would be interesting to relate these strategies more closely. 

When it comes to similarity to our algorithms, there are a variety of approaches based on bounded AND-OR search, such as AO$^*$, LAO$^*$, and LRTDP (e.g., \cite{bonet2003labeled}). Here too, the specification criteria, the correctness bound and the nature of the analysis are largely orthogonal to ours. For example, LAO$^*$ allows  for loops, but the solution is not necessarily a sound and complete N-bounded FSC: it yields a partial policy and does  not allow for arbitrary likelihood scenarios (e.g., an anytime bound such as the goal of generating a FSC where >20\% of the paths reach the goal state). Perhaps one could think of  the difference between LAO$^*$ and LRTDP vs. \textsc{Pandor} as being analogous to the difference between the  paradigms of dynamic programming vs. Monte Carlo methods. In fact, we think our proposal is suitable to combine the best of the two worlds, but that is a topic for future research. We are also excited about the prospect of extending our contributions to the continuous case, and potentially providing asymptotic guarantees for provably correct FSC synthesis.

\section{Conclusions and Discussion} %
\label{sec:discussion_and_conclusions}

In this paper, we presented new theoretical results on a generic technique for synthesizing FSCs in stochastic environments, allowing for highly granular specifications on termination and goal satisfaction. We then proved the soundness and completeness  of that synthesis algorithm. 

As discussed above, the contributions of this paper are solely on the theoretical front.  Nonetheless, %
we will  release  a proof-of-concept implementation of the pseudocode in Alg. \ref{alg:pandor} and~\ref{alg:pandor-helper}.\footnote{The implementation can be accessed at \tt https://github.com/treszkai/pandor} In our preliminary evaluations, we observed that in deterministic domains, our planner has the same runtime as the planner of \cite{Hu2013-and-or}, and the difference is only a small linear factor for additional bookkeeping. But suppose we were to consider %
a noisy variant of the Hall-A domain in Fig.~\ref{fig:hall-a-one}, where every action has a 50-50\% probability of either succeeding or leaving the current state unchanged. Here,  because moving left in cell B is not guaranteed to succeed, we can see that an extra transition is required to move the agent out of B in case the first attempt fails ($\smash{q_1\xrightarrow{\raisebox{-0.5ex}[1ex][0pt]{$\scriptstyle B\ :\ \leftarrow$}}q_0}$). And indeed, the FSC synthesized by {\sc Pandor} contains this transition: \smallskip 

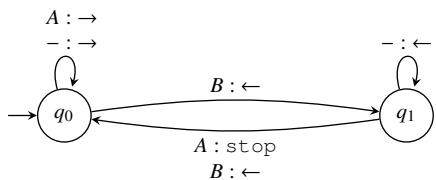
\centerline{\tikz[trim left=0mm,trim right=45mm]{\input{tikz-probhall-a-one-controller.tex}}} 
\smallskip

There are many interesting directions for the future. For example, 
one could investigate: \emph{(a)} effective sampling strategies; \emph{(b)} the tradeoff between higher  \textbf{LGT} bounds vs scalability (i.e., demands on  \textbf{LGT} bounds may be different across applications);  and \emph{(c)} the merits and demerits of the various correctness criteria from the literature for safety-critical applications. To that end,  along with recent advances in the area, we hope that our results provide  theoretical foundations, new proof strategies and a fresh perspective on FSC synthesis in stochastic domains. \smallskip 

%
%

%

%
%

%

%
%
%

%
%
%

%

\input main.bbl
%

%
%

\end{document}

%% file: local-commands.tex
\usepackage{newunicodechar}

\newunicodechar{α}{\alpha}
\newunicodechar{γ}{\gamma}
\newunicodechar{δ}{\delta}
\newunicodechar{ε}{\varepsilon}
\newunicodechar{λ}{\lambda}
\newunicodechar{ω}{\omega}

\newunicodechar{Δ}{\Delta}
\newunicodechar{Π}{\Pi}
\newunicodechar{Ω}{\Omega}

\newunicodechar{→}{\rightarrow}
\newunicodechar{⇒}{\Rightarrow}
\newunicodechar{⟨}{\langle}
\newunicodechar{⟩}{\rangle}
\newunicodechar{≤}{\le}
\newunicodechar{≥}{\ge}

\newunicodechar{∈}{\in}

\newcommand{\LTER}{\mathbf{LTER}}
\newcommand{\LTERPC}{\mathbf{LGT}}
\newcommand{\LGT}{\mathbf{LGT}}

\newcommand{\LTERdesired}{{\mathit{LTER}}^{\star}}
\newcommand{\LTERPCdesired}{{\mathit{LGT}}^{\star}}

\newcommand{\LGTdesired}{{\mathit{LGT}}^{\star}}

\newcommand{\Env}{\mathcal{E}}
\newcommand{\States}{\mathcal{S}}
\newcommand{\Stts}{\mathcal{S}}
\newcommand{\Acts}{\mathcal{A}}
\newcommand{\Actions}{\mathcal{A}}
\newcommand{\Obss}{\mathcal{O}}
\newcommand{\Observations}{\mathcal{O}}
\newcommand{\Goals}{\mathcal{G}}

\newcommand{\Prob}{\mathcal{P}}
\newcommand{\GenProb}{\overline{\mathcal{P}}}

\newcommand{\Xgood}{{\mathrm{good}}}
\newcommand{\Xloop}{{\mathrm{loop}}}
\newcommand{\Xgoal}{{\mathrm{goal}}}
\newcommand{\Xcurrent}{{\mathrm{curr}}}
\newcommand{\Xcurr}{{\mathrm{curr}}}
\newcommand{\Xfail}{{\mathrm{fail}}}
\newcommand{\Xnoter}{{\mathrm{noter}}}

\newcommand{\Xstop}{\text{\texttt{stop}}}

\newcommand{\Xswin}{s_{\mathrm{win}}}
\newcommand{\Xsfail}{s_{\mathrm{fail}}}
\newcommand{\XSP}{\mathit{SP}}

\newcommand{\eg}{e.g., }

\newcommand{\Pandor}{\textsc{Pandor}\ }

\usepackage{amsthm}

\newtheorem{lemma}{Lemma}
\newtheorem{theorem}{Theorem}

\theoremstyle{definition}  
\newtheorem{definition}{Definition}
\newtheorem{problem}{Problem}

\theoremstyle{definition}  
\newtheorem*{notation}{Notation}

\theoremstyle{remark}  

\newcommand{\Xend}{\mathit{end}}
\newcommand{\Xlen}{\mathit{len}}

\newcommand{\given}{\ |\ }
\newcommand{\parm}{\mathord{\color{black!50}\bullet}}

\usepackage{tikz}
\usetikzlibrary{positioning,graphs,automata}
\usetikzlibrary{shapes.multipart}

\tikzset{every picture/.style={line width=0.5pt}}
\tikzset{>=stealth}

\def\XcrossedEdgeFromParent#1#2{
  [style=edge from parent,#1]
  (\tikzparentnode\tikzparentanchor) to #2 node [sloped] {\textsf{x}} (\tikzchildnode\tikzchildanchor)
}

\tikzset{and/.style={draw,circle,fill=black,minimum size=6pt,inner sep=0pt,solid}}
\tikzset{or/.style={draw,circle,fill=white,minimum size=8pt,inner sep=0pt,solid}}
\tikzset{failed node/.style={label=below:{\Lightning}}}
\tikzset{failed child/.style={edge from parent macro=\XcrossedEdgeFromParent}}
\tikzset{unexplored/.style={dashed}}
\tikzset{active/.style={red,very thick}}
\tikzset{myautomata/.style={shorten >=1pt}} 
\tikzset{initial text={}}  
\tikzset{mysplit/.style={rectangle split,rectangle split parts=#1,rectangle split horizontal,draw,rectangle split part align=base,anchor=base,draw=gray}}

\newcommand{\interfuncstrut}{\rule{0pt}{3ex}}

%% file: tikz-hall-a-one.tex
\small

\def\x{6mm}

\begin{scope}[xshift=-2.5*\x]
    \draw (0,0) grid[step=\x] (\x*5,\x);

    \node at (\x/2,\x/2) {\bf A};
    \foreach \i in {1,2,3}
        \node at ({(\i+0.5)*\x},\x/2) {\bf $-$};
    \node at (4.5*\x,0.5*\x) {\bf B};
\end{scope}

\begin{scope}[yshift=-20mm]
    \node[state,initial] at (-20mm,0) (q0) {$q_0$};
    \node[state] at (20mm,0) (q1) {$q_1$};

    \draw[->] (q0) edge[loop above] node[align=center] {$A:\  \rightarrow$\\$- :\ \rightarrow$} (q0);
    \draw[->] (q0) edge[bend left=15] node[above] {$B :\  \leftarrow$} (q1);
    \draw[->] (q1) edge[loop above,xshift=1mm] node {$- :\ \leftarrow$} (q1);
    \draw[->] (q1) edge[bend left=5] node[below,align=left] {$A : \Xstop$} (q0);
\end{scope}

%% file: sec-formalisation.tex

\section{Problem formalisation}

Our contributions do not depend on the details of the formal language (e.g.,  \cite{hu2011-one-dimensional}), and so we consider an abstract framework \citep{Bonet2000-incomplete-info}. 



\begin{definition}
An \emph{environment} \(\Env\) is defined as a tuple \(⟨\States, \Actions, \Observations, \Delta, \Omega⟩ \), whose elements are the following:
 $\States$, $\Actions$, $\Observations$ are finite sets of states, actions, and observations;
    $Δ:\States \times \Actions \rightarrow Π(\States)$ is a stochastic state transition function, where $Π(\States)$ denotes the set of probability distributions over $\States$;
    $Ω: \States → \Observations$ is an observation function.
\end{definition}

A planning problem is defined as an environment, an initial state, and a set of goal states:
\begin{definition}
    A \emph{planning problem} $\Prob$ is a triple $⟨\Env, s_0, \Goals⟩$, where $\Env$ is an environment with state space $\Stts$, $s_0 \in \Stts$ is the \emph{initial state}, and $\Goals \subset \Stts$ is the \emph{set of goal states}.
\end{definition}

We represent loopy plans as follows  \citep{Mealy1955-mealy-machine}:
\begin{definition}
    A \emph{finite state controller} (FSC) $C$ is defined by a tuple \(⟨Q, q_0, \Obss, \Acts, \gamma, \delta ⟩ \), where: $Q = \{q_0, q_1, \ldots, q_{N-1}\}$ is a finite set of \emph{controller states};
       $q_0 \in Q$ is the \emph{initial state} of the controller; 
     $\Obss$ \&  $\Acts$ are the sets of possible observations \& actions; 
        $\gamma: Q \times \Obss \rightarrow (\Acts \cup \{\Xstop\})$ is a partial function called the \emph{labeling function}; 
      $\delta: Q \times \Obss \rightarrow Q$ is a partial function called the \emph{transition function}.

\end{definition}

An FSC forms part of a system $⟨\Env, C⟩$, usually for a planning problem $⟨\Env, s_0, \Goals⟩$. Initially, the environment is in state $s^{(0)} = s_0$, and the FSC $C$ is in controller state $q^{(0)} = q_0$. The controller makes an observation $o^{(0)} = Ω(s_0)$, executes action $a^{(0)} = \gamma(q^{(0)},o^{(0)})$, and transitions to controller state $q^{(1)}= δ(q^{(0)},o^{(0)})$.
The environment transitions to state $s^{(1)} \sim Δ(s^{(1)} | s^{(0)}, a^{(0)})$.
This process is repeated until the special action \verb-stop- is executed, when the state$\rightarrow$observation$\rightarrow$action$\rightarrow$next-state cycle stops.


We call a pair of controller and environment state $⟨q,s⟩$ a \emph{combined state} of the system.

\begin{notation}
We denote the value of any $x \in \{s,q,o,a,p\}$ at step $t$ during the execution of a system by $x^{(t)}$. A sequence is written as $\langle x^{(t)} \rangle_{t=i}^j := \langle x^{(i)}, x^{(i+1)}, \ldots, x^{(j)}\rangle$.
The subsequence of $h = ⟨h^{(t)}⟩_t$ between indices $i$ and $j$ is denoted by $h^{(i:j)} := ⟨h^{(i)}, h^{(i+1)}, \ldots, h^{(j)}⟩$, and $h^{(:j)} := h^{(0:j)}$.
$\Xend(h)$ refers to the last element of $h$.
The concatenation of two \emph{compatible} sequences is denoted by $\cdot$, \eg $h^{(i:j)} \cdot h^{(j:k)} := h^{(i:k)}$. (Sequences $h_1$ and $h_2$ are compatible if $\Xend(h_1) = h_2^{(0)}$.)
\end{notation}


A \emph{history} of a system from a given combined state is one possible sequence of states that it follows, not necessarily until termination.
\begin{definition}
    Let $C = ⟨Q,q_0, \Obss, \Acts,γ,δ⟩$ be a finite state controller, and $\Env$ an environment.
    A \emph{history} $h = ⟨\,⟨q^{(t)}, s^{(t)}⟩\,⟩_{t=0}^T \in (Q \times \Stts)^{<ω}$ of a system $⟨\Env, C⟩$ from the combined state $⟨q^{(0)}, s^{(0)}⟩$ is a finite sequence of combined states such that $p^{(t+1)} = \Delta(s^{(t+1)} \given s^{(t)}, a^{(t)}) > 0$,
      where $a^{(t)} := γ\big(q^{(t)},Ω(s^{(t)})\big)$, and $q^{(t+1)} = δ\big(q^{(t)},Ω(s^{(t)})\big)$, for each $0 \le t < T$.
    A history $h^{(0:T)}$ for a planning problem $⟨\Env, s_0, \Goals⟩$ is \emph{terminating} if $a^{(T)} = \Xstop$.
    A terminating history for a planning problem $⟨\Env, s_0, \Goals⟩$ is a \emph{goal history} if $C$ terminates in a goal state: $a^{(T)} = \Xstop$ and $s^{(T)} \in \Goals$.
    Unless otherwise noted, the first element of a history is $⟨q_0,s_0⟩$.
\end{definition}
Although the action $a^{(t)}$ is not included explicitly in the history, it can be obtained from $⟨q^{(t)}, s^{(t)}⟩$.

The \emph{likelihood} of a history $h$ is the probability that at each step $t$, the environment responds to the controller's action $a^{(t)}$ with the next state $s^{(t+1)}$, and can be defined inductively based on the length of the history:
\begin{align*}
    \ell(h^{(:0)}) 
                := 1,\quad
    \ell(h^{(:t+1)}) 
        := \ell(h^{(:t)}) \cdot Δ(s^{(t+1)} \ |\ s^{(t)}, a^{(t)}).
\end{align*}

The most immediate question here is this: in which sense would we say that a controller is \emph{adequate} for a planning problem? In the absence of noise/nondeterminism, it is easy to show that the transition of combined states is deterministic; put differently, histories can be extended uniquely \citep{Cimatti2003-model-checking,Hu2013-and-or}. So it suffices to argue that there is a \emph{terminating history}  and that it is a \emph{goal history}. Of course, in the presence of nondeterminism, the extension of histories is no longer unique (because of nondeterministic action outcomes), and in the presence of probabilities, it is also useful to consider the likelihood of these extensions. 
We follow  \cite{belle2016-generalized-planning}, where the notion of correctness from \cite{hu2011-one-dimensional,Hu2013-and-or} is extended for noise, and define: 
\begin{definition}
    The \emph{total likelihood of termination} of the system $⟨\Env, C⟩$ on a planning problem $\Prob$ is denoted by $\LTER$:
    \begin{equation}
        \LTER := \!\!\!\!\!\!\sum_{\{h \ |\ h \text{ is a terminating history}\}}\!\!\!\!\!\! \ell(h)
    \end{equation}
\end{definition}
Analogously, for goals, we define:
\begin{definition}
    The \emph{total likelihood of goal termination} of the system $⟨\Env, C⟩$ on a planning problem $\Prob$ is denoted by $\LGT$:
    \begin{equation}
        \LGT := \!\!\!\!\!\!\sum_{\{h \ |\ h \text{ is a goal history}\}}\!\!\!\!\!\! \ell(h)
    \end{equation}
\end{definition}

Finally, we state the search problems we want to solve.
\begin{problem}\label{problem:lgt}
Given a planning problem $\Prob$, an integer $N$, $\LGTdesired \in (0,1)$, find a finite-state controller with at most $N$ states such that $\LGT ≥ \LGTdesired$ for $\Prob$.
\end{problem}
A more fine-grained version is where a minimum bound on $\LTER$ is also possible, defined below: 


\begin{problem}\label{problem:lter-and-lgt}
Given a planning problem $\Prob$, an integer $N$, $\LTERdesired \in (0,1)$, $\LGTdesired \in (0,1)$, find a finite-state controller with at most $N$ states that is $\LTER ≥ \LTERdesired$ and $\LGT ≥ \LGTdesired$ for $\Prob$.
\end{problem}

We restrict our attention to solutions for Problem~1 for the most part, and turn to Problem~2 in a penultimate technical section. 


%% file: sec-HG.tex

\section{Synthesizing classical controllers}

Existing strategies for synthesizing FSCs include the compilation of  generalized planning problems to classical ones  \cite{Bonet2009-Automatic-derivation}, and the generalization of a sequential plan by abstraction \cite{Srivastava2010-Foundations-thesis}.

From the perspective of an {algorithmic schema},  the generic technique of \cite{Hu2013-and-or} is perhaps the simplest to analyze, based on AND-OR search. Here, an environment virtually identical to ours is assumed, and the transition relation is also nondeterministic (but non-probabilistic) via a state transition \emph{relation} $\Delta \subseteq \States \times \Actions \times \States$. The pseudocode for the planner is in Algorithm~\ref{alg:pseudocode-andor}.
Initially, the algorithm starts with the empty controller $C_{ε}$, at the initial controller state $q^{(0)}$, with next states in $S_0$, the initial states of a generalized planning problem.
The {\sc AND-step} function enumerates the outcomes of an action from a given combined state and history, and calls {\sc OR-step} to synthesize a controller that is correct for every outcome.
The {\sc OR-step} function enumerates the extensions of a controller for the current controller state and observation, and thus selects a next action for the current observation, and then calls {\sc AND-step} to test for correctness recursively on the outcomes of the chosen action.\footnote{We provide the pseudocode of the algorithm to easily contrast it with our algorithm, but  some details from \citep{Hu2013-and-or} are omitted for the sake of exposition.} The algorithm is essentially a blind search in controller space, reverting to the last non-deterministic choice point when a branch fails.
The search space is trimmed in two ways. First, when a controller $C$ is found to be not correct, every extension of $C'$ is dropped as well. Second, if $C \prec C'$ (meaning that every controller transition defined by $C$ is the same in $C'$),
then the histories of $C$ that were already explored are not tested again for $C'$.
Most significantly, this exhaustive search results in the  algorithm being \emph{sound} and \emph{complete} \cite{Hu2013-and-or}.




\begin{algorithm}
\begin{algorithmic}[1]
\Require $\GenProb = ⟨\Env, S_0, \Goals⟩$, a generalized planning problem;
\Statex \hspace{5ex}$N$, a bound on the number of controller states.

\Function{\interfuncstrut{}AndOr-synth}{${\GenProb, N}$}
	\State \textbf{return} \Call{AND-step$_{\GenProb, N}$}{$C_{ε}, 0, S_0, ⟨⟩$}
\EndFunction

\Function{\interfuncstrut{}AND-step$_{\GenProb, N}$}{$C,q,S',h$}
    \ForAll{$s' \in S'$}\label{algline:and-step-forall}
        \State $C \gets$ \Call{OR-step$_{\GenProb, N}$}{$C,q,s',h$}
    \EndFor
    \State \textbf{return} $C$
\EndFunction

\Function{\interfuncstrut{}OR-step$_{\GenProb, N}$}{$C,q,s,h$}
    \If{\(s ∈ \Goals\)}
        \State \textbf{return} $C$
    \ElsIf{$⟨q,s⟩ ∈ h$}\label{algline:or-step-repeated}
        \State \textbf{fail}
	\ElsIf{$q \xrightarrow{\Omega(s) / a} q' ∈ C$ for some $q', a$}
    	\State {$S' \gets \{s'\ |\ ⟨s,a,s'⟩ \in Δ \}$}
        \State \textbf{return} \Call{AND-step$_{\GenProb, N}$}{$C', q', S', h\cdot ⟨q, s⟩$}
    \Else
        \State {\bf non-det. branch} $a \in \mathcal A$ and $q' \in \{0,\ldots,N-1\}$\label{algline:or-step-nondet}
        \State $C' \gets C \cup \{q \xrightarrow{\Omega(s) / a} q' \}$
    	\State {$S' \gets \{s'\ |\ ⟨s,a,s'⟩ \in Δ \}$}
        \State \textbf{return} \Call{AND-step$_{\GenProb, N}$}{$C', q', S', h\cdot ⟨q, s⟩$}
    \EndIf
\EndFunction
\end{algorithmic}
\caption{The AND-OR search algorithm for bounded finite state controllers \citep[Fig. 4]{Hu2013-and-or}.}
\label{alg:pseudocode-andor}
\end{algorithm}

%% file: sec-problems.tex

\section{Problems with loops in a noisy environment}\label{sec:problems-noisy-loops}

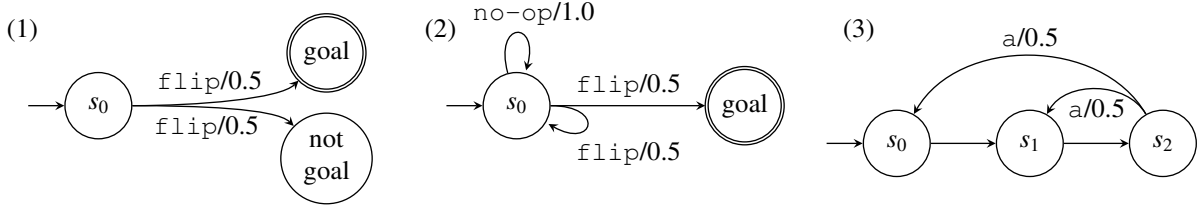
\begin{figure}[t]
\centering
\tikz{\input{tikz-desiderata.tex}}
\caption{Test environments for probabilistic planning with loops.\quad(1)~No~controller with $\LGT=1$ exists.\quad(2)~Difference between decaying and non-decaying loops.\quad(3)~Decaying loops whose combination never terminates.}%
\label{fig:desiderata}
\end{figure}

We identified three desiderata in the introduction, which we justify below.
For \textbf{D1}, it is clear that
there exist planning problems where even the optimal controller might not terminate on every run or end up in a goal state on every terminating run. (Consider a problem with an unavoidable dead end state, for example one where one outcome of a \verb-coin_flip- action ends up in the goal, another outcome results in a dead end state. See \hbox{Fig. \ref{fig:desiderata}.1}.)

\textbf{D2} is the result of assigning probabilities to the different outcomes of an action. If a history repeats a combined state at steps $n$ and $m$, and $\ell(h^{(0:n)}) = \ell(h^{(0:m)})$, then the system will repeat the loop $h^{(n:m)}$ indefinitely, and never terminate. On the other hand, if $\ell(h^{(0:n)}) > \ell(h^{(0:m)})$, and there exists a history $h'$ from $h^{(n)}$ such that $h^{(0:n)} \cdot h'$ terminates in $\Goals$, then $h^{(0:m)} \cdot h'$ will also terminate in $\Goals$. (See \hbox{Fig. \ref{fig:desiderata}.2}.) This means that in a stochastic environment, tracking the likelihood of histories is essential.

For \textbf{D3}, in some environments no looping history has the property that $\ell(h^{(0:n)}) = \ell(h^{(0:m)})$, and yet the system has no terminating runs.
Fig.~\ref{fig:desiderata}.3 illustrates a simple example, where executing action {\tt a} in $s_2$ brings the environment \hbox{either} to $s_0$ or $s_1$, and from $s_0$ the only possible action leads deterministically to $s_1$ and from there back to $s_2$.

It is not possible to analyze every controller synthesis framework in the literature to verify its adherence to  these desiderata, but, in the very least, the case of \citep{Hu2013-and-or} (HD henceforth) is illustrative. Their procedure is correct for the dynamic environment in which they operate, but as can be inferred from the above examples, 
it easily follows that the procedure  
fails to meet the first two of the  desiderata in stochastic environments.
\begin{theorem}
  The algorithm by HD returns with failure if every controller for the planning problem has at least one history that cannot be extended into a goal history.
\end{theorem}

\begin{theorem}
  The algorithm by HD returns with failure if every controller for the  problem has at least one looping history.
\end{theorem}

More significantly, it is not possible to specify likelihood-based correctness criteria, which becomes essential for handling domains where  actions fail, for example, and  meta-level assumptions are unrealistic.

%% file: tikz-desiderata.tex
\def\xshiftfigure{55mm}
\def\xdistancefromedge{10mm}
\def\figheight{30mm}


\tikzset{every edge/.style={->,draw}}

\begin{scope}[xshift=-\xshiftfigure]

    \node at (5mm,-5mm) {(1)};

    \node[state,initial] (s_0) at (1.5*\xdistancefromedge,-15mm) {$s_0$};

    \node[state,accepting] (s_goal) at (\xshiftfigure-\xdistancefromedge,-8mm) {goal};
    \node[state] (s_nongoal) at (\xshiftfigure-\xdistancefromedge,-22mm) [align=center,inner sep=1mm] {not\\goal};

    \draw (s_0) edge[out=0,in=225,in looseness=0.4] node [midway,yshift=1.5ex,xshift=-2ex] {\texttt{flip}/0.5} (s_goal);
    \draw (s_0) edge[out=0,in=135,in looseness=0.4] node [midway,swap,yshift=-1.5ex,xshift=-2ex] {\texttt{flip}/0.5} (s_nongoal);

\end{scope}

\begin{scope}[xshift=0cm]
    \node at (5mm,-5mm) {(2)};

    \node[state,initial] (2-0) at (1.5*\xdistancefromedge,-\figheight/2) {$s_0$};

    \node[state,accepting] (2-goal) at (\xshiftfigure-\xdistancefromedge,-\figheight/2) {goal};

    \draw (2-0) edge[loop above] node[above right,xshift=-7mm] {{\tt no-op}/1.0} (2-0);
    \draw (2-0) edge node[above] {\texttt{flip}/0.5}
                         node[above,yshift=-9mm] {\texttt{flip}/0.5} (2-goal);
    \draw (2-0) edge[out=0,in=-30,loop] (2-0);

\end{scope}

\begin{scope}[xshift=\xshiftfigure]
    \node at (5mm,-5mm) {(3)};

    \node[state,initial] (3-0) at (\xdistancefromedge,-\figheight*2/3) {$s_0$};
    \node[state] (3-1) at (\xshiftfigure/2,-\figheight*2/3) {$s_1$};
    \node[state] (3-2) at (\xshiftfigure-\xdistancefromedge,-\figheight*2/3) {$s_2$};

    \draw (3-0) edge node[below] {} (3-1);
    \draw (3-1) edge node[below] {} (3-2);
    \draw (3-2) edge[out=120,in=60] node[below] {{\tt a}/0.5} (3-1);
    \draw (3-2) edge[out=120,in=60] node[above] {{\tt a}/0.5} (3-0);
\end{scope}

%% file: sec-pandor.tex
\input algo1
\input algo2


\section{Algorithm for loopy planning}

We propose a search algorithm that provably meets all three of the desiderata. It also instantiates an AND-OR search in that it simulates the runs of a system, enumerates the possible controller extensions whenever it reaches an undefined action, and when an action has multiple outcomes, it does a depth-first search on the next states recursively. However, it fixes the shortcomings stated in the previous section: instead of only allowing controllers that are correct on every run, it synthesizes controllers whose correctness likelihood exceeds some likelihood given as input to the algorithm; and it is capable of handling looping histories. As it is a probabilistic variant of the AND-OR search, we name it {\sc Pandor}.

\subsection{Allowing less than perfect controllers}
The basic idea behind our planner is that it maintains an upper and lower bound for the $\LTERPC$ of the current controller, based on the histories simulated thus far. Whenever a failing run is encountered, the upper bound is decreased by the likelihood of this run; similarly, a goal run increases the lower bound on $\LTERPC$. When the lower bound exceeds the desired correctness likelihood (hereafter denoted by $\LTERPCdesired$), the current controller is guaranteed to be ``good enough'', and the algorithm returns with success. When the upper bound is lower than $\LTERPCdesired$, none of the extensions of the controller is sufficiently good, and we revert the program state to the point of the last non-deterministic choice point. In the simplest variant of {\sc Pandor}, any run with repeated combined states is counted as a failed run, thus it meets \textbf{D1} but not \textbf{D2}. This property leads to an underestimation of the lower and upper bounds on $\LTERPC$, making the search sound but incomplete.

While the planner of \cite{Hu2013-and-or} declared a controller and all of its extensions insufficiently good when it had a single failing run, we relax this condition.
Now a controller is insufficiently good when the total likelihood of all of its failing runs exceeds $1 - \LTERPCdesired$; this results in the same behavior as that of its predecessor when $\LTERPCdesired = 1$. 

\subsection{Correctly counting looping histories}
In order to account for looping histories, we draw on the following insight.
Suppose that for a history $h^{(0:k)}$, there is a history $h_{\Xloop}$ from $h^{(k)}$ with $\Xend(h_{\Xloop}) = h^{(k)}$, and another history $h_{\Xgoal}$ from $h^{(k)}$ that terminates in a goal state.
A \emph{system} with an FSC has the Markov property such that both the next action of the controller and the next state of the environment are defined by the current combined state. As a result, \hbox{$h \cdot h_\Xgoal$}, \hbox{$h \cdot h_\Xloop \cdot h_\Xgoal$}, \hbox{$h \cdot h_\Xloop \cdot h_\Xloop \cdot h_\Xgoal$} and so on are all valid goal histories, where the one with $m$ repetitions of $h_\Xloop$ has likelihood $\ell(h) \big(\ell(h_\Xloop)\big)^m \ell(h_\Xgoal)$.
 These likelihoods form a geometric progression, whose sum for all $m\ge0$ is $\ell(h) \ell(h_\Xgoal) / \big(1 - \ell(h_\Xloop)\big)$. (The existence of two distinct histories from $h^{(k)}$, namely $h_\Xgoal$ and $h_\Xloop$, guarantees that $\ell(h_\Xloop) < 1$.) In the following, we describe how to utilize this argument.

We said that {\sc Pandor} enumerates the histories of FSCs; let $h_{\Xcurrent}^{(0:n)}$ be the currently simulated history at some point during execution. Now we construct the set of all goal histories from pairwise disjoint sets of histories, one set for each $0 \le k \le n$.

Denote by $H_{\Xloop}^{k}$ the set of histories $h_{\Xloop}$ from $h_{\Xcurrent}^{(k)}$ with the following properties:
\begin{list}{}{} 
	\item[\quad\bf L1.] $\smash{h_{\Xloop}^{(0)} = h_{\vphantom{l}\Xcurrent}^{(k)}}$
	\item[\quad\bf L2.] $\Xend(h_{\Xloop}) = \smash{h_{\Xcurrent}^{(k)}}$
	\item[\quad\bf L3.] apart from its first and last element, $h_\Xloop$ doesn't\\
	 \phantom{\bf L3.} contain ${h_{\Xcurrent}^{(k)}}$, and
	\item[\quad\bf L4.] no element of $h_\Xloop$ is equal to $\smash{h_{\Xcurrent}^{(i)}}$ for any $i < k$.
\end{list}
Furthermore, let $H_{\Xgoal}^{k+1}$ denote the set of histories $h_{\Xgoal}$ from $h_{\Xcurrent}^{(k)}$ with the following properties: 
\begin{list}{}{} 
	\item[\quad\bf W1.] ${h_{\Xgoal}^{(0)} = h_{\vphantom{l}\Xcurrent}^{(k)}}$
	\item[\quad\bf W2.] $\Xend(s_{\Xgoal}) \in \Goals$, and $\smash{γ\big(Ω(\Xend(h_{\Xgoal}))\big) = \Xstop}$
	\item[\quad\bf W3.] apart from its first element, $h_\Xgoal$ doesn't contain\\ \phantom{{\bf W3.} }$h_{\Xcurrent}^{(k)}$,
	\item[\quad\bf W4.] no element of $h_\Xgoal$ is equal to $\smash{h_{\Xcurrent}^{(i)}}$ for any $i < k$,
	\item[\quad\bf W5.] $\smash{h_{\Xgoal}^{(1)} \neq h_{\vphantom{l}\Xcurrent}^{(k+1)}}$.
\end{list}

Condition {\bf W2} states that $h_\Xgoal$ is a goal history.
%
%
The other conditions collectively ensure that if
$h^\star$ is a goal history from $h^{(k-1)}$ such that $h^{\star\,(1:)}$ doesn't repeat any element of $h^{(:k-1)}$, then it is either {\bf W1}--{\bf W5} for $k-1$, or it can be uniquely pieced together by elements of $H_\Xloop^k$ and $H_\Xgoal^{k+1}$.
\begin{lemma}
Let $h^\star$ meet {\bf W1}--{\bf W4} for some $k$. Then either of these hold, but not both:
\begin{itemize}\raggedright
	\item $h^\star \in H_\Xgoal^{k}$, or
	\item there is a unique $m ≥ 0$ such that for some
	$h_{\Xloop,i} \in H_\Xloop^{k}$ for all $0 \le i \le m$
	and some $h_{\Xgoal} \in H_\Xgoal^{k+1}$, we have:
\end{itemize}
\begin{align}
	\smash{h^{\star} = h_\Xcurr^{(k-1:k)} \cdot h_{\Xloop,1} \cdot \ldots \cdot h_{\Xloop,m} \cdot h_{\Xgoal}.}
\end{align}
\end{lemma}

Let $\tilde{\alpha}_{\Xgoal}^{(k)}$ and $\tilde{\lambda}_{\Xloop}^{(k)}$ denote the sum of the likelihood of the elements of $H_\Xgoal^{(k)}$ and $H_\Xloop^{(k)}$.
\begin{align} 
    \tilde{\alpha}_{\Xgoal}^{(k)} &= \sum_{h_{\Xgoal} \in H_\Xgoal^{(k)}} \ell(h_\Xgoal) \\
    \tilde{\lambda}_{\Xloop}^{(k)} &= \sum_{h_{\Xloop} \in H_\Xloop^{(k)}} \ell(h_\Xloop)
\end{align}

The likelihood of the system terminating in a goal state (s.t.i.g.) if started in the initial state $h_{\Xcurr}^{(0)}$ can be calculated inductively, as follows.
\begin{lemma} 
    Let $\tilde{\lambda}_{\Xgoal}^{(k)}$ denote the probability of s.t.i.g.\ after the history $h_\Xcurr^{(:k-1)}$ without repeating any combined state of $h_\Xcurr^{(:k-1)}$. Then the following hold for any $0 \le k < n$:
    \begin{align}
        \tilde{\lambda}_{\Xgoal}^{(n)} &= \tilde{\alpha}_\Xgoal^{(n)}, \\
        \tilde{\lambda}_{\Xgoal}^{(k)} &= p^{(k)} \tilde{\lambda}_\Xgoal^{(k+1)} / (1 - \tilde{\lambda}_\Xloop^{(k)}) + \tilde\alpha_{\Xgoal}^{(k)},
    \end{align}
    where $p^{(k)}$ is the probability of transitioning from $h_{\Xcurr}^{(k-1)}$ to $h_\Xcurr^{(k)}$, i.e.
    \begin{equation}
        p^{(k)} = Δ\big(s_\Xcurr^{(k)} \ \big|\ s_\Xcurr^{(k-1)}, δ(q_\Xcurr^{(k-1)}, Ω(h_\Xcurr^{(k-1)})) \big).
    \end{equation}
\end{lemma}
This is because $\tilde{\lambda}_{\Xgoal}^{(k)}$ is equal to
the probability of s.t.i.g.\ from $h_\Xcurr^{(:k)}$ without repeating $h_\Xcurr^{(:k-1)}$\ \ plus
the probability of s.t.i.g.\ from $h_\Xcurr^{(:k-1)}$ without repeating $h_\Xcurr^{(:k-1)}$ if the $k$th combined state is not equal to $h_\Xcurr^{(k)}$. The latter is simply $\tilde\alpha_{\Xgoal}^{(k)}$, and the former is the sum of an infinite geometric series whose ratio is $\tilde\lambda^{(k)}_\Xloop$.

It is easy to see that
 with the natural definition of \hbox{$h^{(:-1)} := ⟨⟩$}, the probability of s.t.i.g.\ from $h_\Xcurr^{(0)}$ is $\LTERPC$.
\begin{lemma}\label{lemma:LGT-lambda-goal}
    $\LGT = \tilde{\lambda}_{\Xgoal}^{(0)}$
\end{lemma}

\subsection{Measuring the looping goal likelihood}

While {\sc Pandor} simulates the histories of a controller,
it maintains variables $α_{\Xgoal}^{(k)}$ for each $k$, which is the sum of likelihoods of the elements $h \in H_\Xgoal^{(k)}$ such that $h_{\Xcurr}^{(k-1)} \cdot h$ has already been visited. (This is done by increasing the relevant $α_\Xgoal^{(k)}$ in the OR-step when the controller terminates in a goal state.)
The following inequality trivially holds:
\begin{lemma}\label{lemma:alpha-goal-inequality}
$α_{\Xgoal}^{(k)} \le \tilde\alpha_{\Xgoal}^{(k)}$, with equality when every non-looping goal history from $h_\Xcurr^{(k-1)}$ has been visited.
\end{lemma}

When a looping history is found in an OR-step, i.e., $⟨q, s⟩ = h_\Xcurr^{(k)}$ for some $k\le \Xlen(h_\Xcurr)$, such that the likelihood of the loop is $\ell(h_\Xcurr^{(k:n)}) \cdot p = p_\Xloop < 1$, then $α_{\Xloop}^{(k,m)}$ is increased by $p_\Xloop$. (Note that in the {\sc OR-step} function, $h_\Xcurr$ doesn't yet include the current combined state.)
The intuitive meaning of $α_{\Xloop}^{(k,m)}$ is the probability estimate of looping to $h^{(k)}$ after a history of $h^{(:m)}$ with not $h^{(m+1)}$ as the $(m+1)$-st step, without repeating any of $h^{(:m-1)}$ and repeating $h^{(m)}$ exactly once (Fig.~\ref{fig:loops-in-loops}).
Now we can calculate $λ_\Xloop^{(k)}$: we can loop to step $k$ with not $h^{(k+1)}$ as the $(k+1)$-st step, or with $h^{(k+1:m)}$ as the next steps, repeating $h^{(m)}$ any number of times but then looping back to $h^{(k)}$ from $h^{(m)}$ without $h^{(m+1)}$ as the next step -- for any $k < m \le n$.
\begin{align}
    λ_{\Xloop}^{(k)}
        &= α_\Xloop^{(k,k)} + \\
        &\quad+ p^{(k+1)} \ \big(1 - λ_\Xloop^{(k+1)}\big)^{-1}\  α_\Xloop^{(k,k+1)} + \ldots + \notag\\
        &\quad+ p^{(k+1)} p^{(k+2)} \cdots p^{(n)} \ \big(1 - λ_\Xloop^{(n)}\big)^{-1}\ α_\Xloop^{(k,n)}. \notag
\end{align}
This calculation is done in \hbox{lines \ref{algline:calc-lambda-loop-for-begin}--\ref{algline:calc-lambda-loop-for-end}} of Alg.~\ref{alg:pandor-helper}.
The following result is proved easily by induction on $k$.
\begin{lemma}\label{lemma:lambda-loop-inequality}
$λ_\Xloop^{(k)} \le \tilde\lambda_\Xloop^{(k)}$, with equality when every once-looping history from $h_\Xcurr^{(k)}$ has been visited.
\end{lemma}

Using Lemmas \ref{lemma:alpha-goal-inequality}~and~\ref{lemma:lambda-loop-inequality}, the following can be seen:

\begin{lemma}\label{lemma:lambda-goal-inequality}
$λ_\Xgoal^{(k)} \le \tilde\lambda_\Xgoal^{(k)}$, with equality when every non-looping goal history from $h_\Xcurr^{(k-1)}$ has been visited.
\end{lemma}

In Alg.~\ref{alg:pandor-helper}, the {\sc CalcLambda} function calculates $λ_\Xgoal^{(0)}$ based on $\bm{\alpha}_\Xgoal$, $\bm{\alpha}_\Xloop$, and $h_\Xcurr$, and this $λ_\Xgoal^{(0)}$ serves as the basis for termination in the AND-step.

\subsection{Failing and non-terminating histories}
We can treat \emph{failing histories} (histories that terminate in a non-goal state) and histories that contain a non-decaying loop similarly to goal histories. We account for them via $α_\Xfail^{(k)}$ and $α_\Xnoter^{(k)}$ values in a manner analogous to  $α_\Xgoal^{(k)}$. 

We have similar results for the relevant $λ$ values as before:
\begin{lemma}\label{lemma:lambda-fail-inequality}
$λ_\Xfail^{(k)} \le \tilde\lambda_\Xfail^{(k)}$, with equality when every non-looping failing history from $h_\Xcurr^{(k-1)}$ has been visited.
\end{lemma}

Calculating $λ_\Xnoter^{(k)}$ is peculiar in that multiple decaying loops can add up to a a history that cannot be extended into a terminating history (Fig.~\ref{fig:desiderata}.3).
When no history from $h_\Xcurr^{(k)}$ terminates, $\tilde{λ}_\Xloop^{(k+1)} + \tilde{λ}_\Xnoter^{(k)} = 1$ – for example, $h_\Xcurr^{(k)}$ has an extension with likelihood $0.1$ with a non-decaying loop at the end, and another extension with likelihood $0.9$ loops back to $h_\Xcurr^{(k)}$.
When this happens, the values of $α_\Xloop^{(k:,k:)}$ are zeroed out, and $λ_\Xnoter^{(k)}$ is assigned $p^{(k)}$. (Lines~\ref{algline:loop-noter-begin}--\ref{algline:loop-noter-end}.) With this caveat, the inequality result for $λ_\Xnoter$ is the following:
\begin{lemma}\label{lemma:lambda-noter-inequality}
$\smash{λ_\Xnoter^{(k)} \le \tilde\lambda_\Xnoter^{(k)}}$, with equality when every once-looping history from $h_\Xcurr^{(k-1)}$ has been visited.
\end{lemma}

We now have classified histories as those that could be extended into terminating ones (either in a goal state or not) and those that have a non-decaying loop.
\begin{lemma}\label{lemma:lambda-lter}
        $\smash{\tilde{λ}_\Xgoal^{(0)} + \tilde{λ}_\Xfail^{(0)} = \LTER = 1 - \tilde{λ}_\Xnoter^{(0)}}$
\end{lemma}


\subsection{Rolling up the $\alpha$ values}
Next, we see what should happen to the $\smash{\alpha_{\parm}^{(\parm)}}$ values when the current history changes. Clearly, when $h_\Xcurr$ is extended, $α_\Xgoal$ should be extended with an additional zero item, and $α_\Xloop$ with an additional row\&column of zeros (lines \hbox{\ref{algline:alpha-goal-gets-zero}--\ref{algline:alpha-loop-gets-zero}}).

When $h_\Xcurr$ is shortened when the {\sc AND-step} function returns, the last element (or last row\&column) of these variables needs to be integrated to the previous ones before deleting them (\hbox{\sc CumulateAlpha} function at lines~\ref{algline:and-step-call-cumulate},~\ref{algline:cumulate-alpha}).
The approach is similar to how the first iteration of $λ_\Xgoal$ and $λ_\Xloop$ was calculated: in fact, the new $α$ values are chosen so that {\sc CalcLambda} returns the same values before and after {\sc CumulateAlpha} is called.
It is important to note that this change in the $α$ values doesn't affect our earlier results:
\begin{lemma}
    After $\smash{α_{\parm}^{(\parm)}}$ is assigned the values returned by \hbox{\sc CumulateAlpha}, Lemmas \ref{lemma:lambda-loop-inequality}, \ref{lemma:lambda-goal-inequality}, \ref{lemma:lambda-fail-inequality}, \ref{lemma:lambda-noter-inequality} still hold.
\end{lemma}

\subsection{Correctness of the search}
Our final theorem states that {\sc Pandor} meets desiderata {\bf D1}--{\bf D3}, and solves Problem~\ref{problem:lgt} correctly. 

\begin{theorem}\label{thm:correctness-lgt}
    Given a planning problem $\Prob$, integer $N$, and $\LGTdesired \in (0,1)$, the search algorithm {\sc Pandor} is \emph{sound} and \emph{complete:}
    every FSC $C$ returned by {\sc Pandor-synth} is $N\!$-bounded and $\LGT \ge \LGTdesired$ for $\Prob$, and
    if there exists an $N\!$-bounded controller that is $\LGT \ge \LGTdesired$ for $\Prob$, then one such FSC will be found.
\end{theorem}

\begin{proof}[Proof sketch]
\textbf{Soundness.}
A controller $C$ is returned by {\sc Pandor-synth} only if $\LGTdesired \le λ_\Xgoal^{(0)}$ \\ (line~2.\ref{algline:and-goal-condition}).
By Lemmas \ref{lemma:LGT-lambda-goal}~and~\ref{lemma:lambda-goal-inequality}:
\begin{equation}
    λ_\Xgoal^{(0)} ≤ \tilde{λ}_\Xgoal^{(0)} = \LGT,
\end{equation}
making the controller $\LGTdesired ≤ \LGT$ for $\Prob$.

\textbf{Completeness.}
Suppose there exists an $N\!$-bounded controller $C_{\Xgood}$ which is $\LGT ≥ \LGTdesired$ for $\Prob$.

Suppose a smaller controller $C' \prec C_\Xgood$ is rejected (property $\dagger$).
A failing or non-terminating history of $C'$ has the same property for $C_\Xgood$ as well, and is a valid history for the system $⟨\Env, C_\Xgood⟩$ (property $\ddagger$).
A failing or non-terminating history does not terminate in a goal state (property $\star$).
\begin{align}\label{eq:completeness-lgt}
\LGT&(C_\Xgood) ≤ & \\
  &≤ 1 - \tilde{λ}_\Xfail^{(0)}(C_\Xgood) - \tilde{λ}_\Xnoter^{(0)}(C_\Xgood) &\quad\text{by $\star$} \notag\\
  &≤ 1 - \tilde{λ}_\Xfail^{(0)}(C') - \tilde{λ}_\Xnoter^{(0)}(C') &\quad\text{by $\ddagger$} \notag \\
  &≤ 1 - λ_\Xfail^{(0)}(C') - λ_\Xnoter^{(0)}(C') &\quad\text{Lemma \ref{lemma:lambda-fail-inequality}, \ref{lemma:lambda-noter-inequality}} \notag \\
  &< \LGTdesired &\quad\text{by $\dagger$} \notag
\end{align}
This is against the premise that $C_\Xgood$ is $\LGT ≥ \LGTdesired$, contradiction: no smaller controller is rejected.

Suppose that when the current controller is $C' \prec C_\Xgood$, at a non-deterministic choice the next controller $C''$ is such that $C' \prec C'' \not\preceq C_\Xgood$.
If this execution branch of $C''$ does not fail, then a controller was returned, which was $\LGT ≥ \LGTdesired$ by the soundness of the search.

In a finite environment, a system with an FSC has finitely many combined states. This implies that the number and length of the at-most-once-looping histories is bounded above, and so is the number of OR-steps required to explore these histories. At the end of the execution, every such history of $C$ has been simulated, resulting in $λ_\Xgoal^{(0)} + λ_\Xfail^{(0)} +\penalty-100 λ_\Xnoter^{(0)} = 1$ by Lemmas \ref{lemma:lambda-goal-inequality}, \ref{lemma:lambda-fail-inequality}, \ref{lemma:lambda-noter-inequality}, and~\ref{lemma:lambda-lter}. At this time one of the termination conditions is fulfilled.

If at every non-deterministic choice, $C''$ is chosen such that $C' \prec C'' \preceq C_\Xgood$, then as $C''$ can't be rejected (by the argument above), either $C''$ or an extension of it will be returned.
\end{proof}

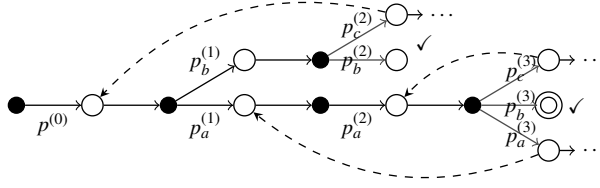
\begin{figure}[t]\centering
    \tikz{\input{tikz-loops-in-loops.tex}}
	 \caption{An example AND-OR tree corresponding to the execution of {\sc Pandor}, with the root at the left. Numbers on the edges are the transition probabilities; filled black circles: AND node; empty circles: OR node; a checkmark: terminating in a goal state; a dashed arrow indicates that the relevant states are equal in a looping history, with an infinite tree below.
	  When all nodes are explored and the double circle is the current node, the non-zero $α$ values are $α_\Xgoal^{(1)} = p_b^{(1)} p_b^{(2)}$,\ \ $α_\Xgoal^{(3)} = p_b^{(3)}$,\ \  $α_\Xloop^{(0,0)} = p_b^{(1)} p_c^{(2)}$,\ \  $α_\Xloop^{(1,3)} = p_a^{(2)} p_a^{(3)}$,\ \  $α_\Xloop^{(2,3)} = p_c^{(3)}$.}
	    \label{fig:loops-in-loops}
		\end{figure}


\subsection{Planning for minimum likelihood of termination}

The above results outline how our algorithm can plan for a minimum $\LGT$, but only minor modifications are required for setting a lower bound on $\LTER$ as well. Formally, 

\begin{problem}\label{problem:lter-and-lgt}
Given a planning problem $\Prob$, an integer $N$, $\LTERdesired \in (0,1)$, $\LGTdesired \in (0,1)$, find a finite-state controller with at most $N$ states that is $\LTER ≥ \LTERdesired$ and $\LGT ≥ \LGTdesired$ for $\Prob$.
\end{problem}

To solve this, the only change required in the pseudocode is in lines~\ref{algline:and-goal-condition}--\ref{algline:and-goal-condition-end}, where we simply extend the criteria for early termination and failure. These changes are shown in Alg.~\ref{alg:pandor-lter}.

\begin{algorithm}[ht]
\begin{algorithmic}[1]
\Require ${\Prob}$, $N$, $\LTERPCdesired$: as before,
\Statex \hspace{5ex}$\LTERdesired$: the desired minimum $\LTER$

\setcounter{ALG@line}{10}
\Statex
\State {\bf if} $λ_{\Xgoal} \ge \LTERPCdesired$ and $λ_{\Xgoal} + λ_{\Xfail} \ge \LTERdesired$
\State \hspace{1.5em} \textbf{return} $C$
\State {\bf else if} $1 - λ_{\Xfail} - λ_{\Xnoter}\! <\! \LTERPCdesired$ or $1 - λ_{\Xnoter} \!<\! \LTERdesired$
\State \hspace{1.5em} \textbf{fail} this non-deterministic branch
\State {\bf end if}
\end{algorithmic}
\caption{Changes required in Algorithm~2 to specify a lower bound on both $\LTER$ and $\LTERPC$.}\label{alg:pandor-lter}
\end{algorithm}

This new search process is sound and complete with respect to Problem~\ref{problem:lter-and-lgt}.
\begin{theorem}
    Given a planning problem $\Prob$, integer $N$, $\LTERdesired \in (0,1)$, and $\LGTdesired \in (0,1)$, the search algorithm {\sc Pandor} is \emph{sound} and \emph{complete}:
    every FSC $C$ returned by {\sc Pandor-synth} is $N\!$-bounded and $\LTER \ge \LTERdesired$ and $\LGT \ge \LGTdesired$ for $\Prob$, and
    if there exists an $N\!$-bounded controller that is $\LTER \ge \LTERdesired$ and $\LGT \ge \LGTdesired$ for $\Prob$, then one such FSC will be found.
\end{theorem}

\begin{proof}[Proof sketch]
\textbf{Soundness.}
A controller $C$ is returned by {\sc Pandor-synth} only if $\LGTdesired \le λ_\Xgoal^{(0)}$ and $\LTERdesired \le λ_\Xgoal^{(0)} + λ_\Xfail^{(0)}$ (line~4.\ref{algline:and-goal-condition}).
By Lemmas \ref{lemma:LGT-lambda-goal},~\ref{lemma:lambda-goal-inequality}, and \ref{lemma:lambda-fail-inequality},~\ref{lemma:lambda-lter}:
\begin{align}
    λ_\Xgoal^{(0)} &≤ \tilde{λ}_\Xgoal^{(0)} = \LGT,\\
    λ_\Xgoal^{(0)} + λ_\Xfail^{(0)} &≤ \tilde{λ}_\Xgoal^{(0)} + \tilde{λ}_\Xfail^{(0)} = \LTER,
\end{align}
making the controller $\LGTdesired ≤ \LGT$ and $\LTERdesired ≤ \LTER$ for $\Prob$.

\textbf{Completeness.}
Suppose there exists an $N\!$-bounded controller $C_{\Xgood}$ which is $\LGT ≥ \LGTdesired$ and $\LTER ≥ \LTERdesired$ for $\Prob$.

Suppose a smaller controller $C' \prec C_\Xgood$ is rejected -- this can happen either for not meeting the bound on $\LGTdesired$ (property $\dagger$) or on $\LTERdesired$ ($\dagger\dagger$).
A failing or non-terminating history of $C'$ has the same property for $C_\Xgood$ as well, and is a valid history for the system $⟨\Env, C_\Xgood⟩$ (property $\ddagger$).
A failing or non-terminating history does not terminate in a goal state (property $\star$).
If the controller is rejected for $\dagger$, then inequality~\ref{eq:completeness-lgt} holds, otherwise:
\begin{align}\label{eq:completeness-lter}
\LTER(C_\Xgood) 
  &≤ 1 - \tilde{λ}_\Xnoter^{(0)}(C_\Xgood) &\quad\text{by $\star$} \\
  &≤ 1 - \tilde{λ}_\Xnoter^{(0)}(C') &\quad\text{by $\ddagger$} \notag \\
  &≤ 1 - λ_\Xnoter^{(0)}(C') &\text{Lemma \ref{lemma:lambda-lter}} \notag \\
  &< \LTERdesired &\quad\text{by $\dagger\dagger$} \notag
\end{align}
Either Eq.~\ref{eq:completeness-lgt} or Eq.~\ref{eq:completeness-lter} is against the premise that $C_\Xgood$ is $\LGT ≥ \LGTdesired$ and $\LTER ≥ \LTERdesired$, contradiction: no smaller controller is rejected.

The algorithm terminates for the reasons described in the proof of Theorem~3.

Suppose that when the current controller is $C' \prec C_\Xgood$, at a non-deterministic choice the next controller $C''$ is such that $C' \prec C'' \not\preceq C_\Xgood$.
If this execution branch of $C''$ does not fail, then a controller was returned, which was $\LGT ≥ \LGTdesired$ and $\LTER ≥ \LTERdesired$ by the soundness of the search.

If at every non-deterministic choice, $C''$ is chosen such that $C' \prec C'' \preceq C_\Xgood$, then as $C''$ can't be rejected, either $C''$ or an extension of it will be returned.
\end{proof}


\subsection{Time and space complexity}

In order to explore the whole environment with a given controller, we need to take $O(b^{h_{\text{max}}})$ steps, where 
$b$ is the branching factor at the AND-step (the maximum number of outcomes of an action), and
$h_{\text{max}}$ is the length of the longest possible history without a repeated state (i.e. $h_{\text{max}} \leq |\States|$).
At every AND-step, \Pandor needs to calculate the $\lambda$ vectors, which takes $O(h_{\text{max}}^2)$ steps due to the size of $\alpha_{\text{loop}}$.
This exploration needs to be done, usually to different depths, for every possible non-isomorphic $N$-bounded FSC, which we denote by $\#_C$.
A controller is defined by its transitions, hence $\#_C < (N\cdot|\Observations|)^{N\cdot|\Actions|}$.
It follows that the time complexity of the algorithm is $O(b^{h_{\text{max}}} \cdot \#_C \cdot h_{\text{max}}^2) = O(b^{h_{\text{max}}} \cdot \#_C$).

These numbers are realized in an adversarial environment with extremely low probability (depending on the action/outcome selection); in most realistic situations (i.e., non-adversarial environments), failure/success would be orders of magnitude quicker.
First, whenever a controller is found to not meet the desired $\LTERPCdesired$, all of its extensions are discarded immediately, leaving us with $\#_C \ll (N\cdot|\Observations|)^{N\cdot|\Actions|}$. 
Secondly, most controllers are unable to explore the whole environment, and require orders of magnitude fewer than $b^{h_{\text{max}}}$ steps. We believe the search process could be further improved using heuristics.

Analogously, we need to store the alpha vectors and the $\alpha_{\text{loop}}$ matrix at each controller extension, which each require $O(h_{\text{max}}^2)$ space.
If the maximum number of controller transitions is $\#_T$ (where $\#_T \leq N \cdot |\Observations|$), then this results in a space complexity of $O(h_{\text{max}}^2 \cdot \#_T)$.

%% file: algo1.tex
\begin{algorithm} 
\begin{algorithmic}[1]
\Require ${\Prob} = ⟨\Env, {s_0}, \Goals⟩$, a planning problem;
\Statex \hspace{5ex}$N$, a bound on the number of controller states;
\Statex \hspace{5ex}$\LTERPCdesired$: the desired minimum $\LTERPC$.
\Function{\interfuncstrut{}Pandor-synth}{$\Prob, N$}
	\State (global) $\bm{\alpha} \gets ⟨⟩$
	\State \textbf{return} \Call{AND-step$_{\Prob, N}$}{$C_{ε}, 0, \{⟨s_0, 1.0⟩\}, ⟨⟩$}\label{algline:call-and}
\EndFunction
\Function{\interfuncstrut{}AND-step$_{\Prob, N}$}{$C,q,\XSP',h^{(0:n)}$}
	\State $α^{(n+1)}_x \gets 0$, \quad for $x \in \{\Xgoal, \Xfail, \Xnoter\}$\label{algline:alpha-goal-gets-zero}
	\State $α^{(n+1,:)}_{\Xloop} \gets 0$; \quad $α^{(:,n+1)}_{\Xloop} \gets 0$\label{algline:alpha-loop-gets-zero}
    \ForAll{${⟨s',p'⟩} \in \XSP'$}
        \State $C \gets$ \Call{OR-step$_{\Prob, N}$}{$C,q,s',p',h$}
        \State $\bm{λ} \gets$ \Call{CalcLambda}{$h,\bm{\alpha}$}
		\If{$λ_{\Xgoal} \ge \LTERPCdesired$}\label{algline:and-goal-condition}
			\State \textbf{return} $C$
		\ElsIf{$1 - λ_{\Xfail} - λ_{\Xnoter} < \LTERPCdesired$}
			\State \textbf{fail} this non-deterministic branch
		\EndIf\label{algline:and-goal-condition-end}
    \EndFor
    \State $\bm{\alpha} \gets$ \Call{CumulateAlpha}{$h, \bm{\alpha}$}\label{algline:and-step-call-cumulate}
    \State \textbf{return} $C$
\EndFunction

\Function{\interfuncstrut{}OR-step$_{\Prob, N}$}{$C,q,s{,p},h^{(0:n)}$}
    \If{{$s = \Xswin$}}
		\State $α^{(n+1)}_{\Xgoal} \gets α^{(n+1)}_{\Xgoal} + p$;\quad{\bf return} $C$
    \ElsIf{{$s = \Xsfail$}}
		\State $α^{(n+1)}_{\Xfail} \gets α^{(n+1)}_{\Xfail} + p$;\quad{\bf return} $C$
	\ElsIf{$q^{(k)} = q$ and $s^{(k)} = s$ for some $k$}
	    \If{$p = 1$ and $p^{(i)} = 1$ for all $k+1 \le i \le n$}
	        \State $α^{(n+1)}_{\Xnoter} \gets 1$;\quad{\bf return} $C$
	    \Else
	        \State $α^{(k,n)}_{\Xloop} \gets α^{(k,n)}_{\Xloop} + p$;\quad{\bf return} $C$
	    \EndIf
	\ElsIf{$q \xrightarrow{\Omega(s) / a} q' \in C$ for some $a, q'$}
		\State $\XSP' \gets$ \Call{NextStates$_{\Prob}$}{$s,a$}
    	\State {\bf return} \Call{AND-step$_{\Prob, N}$}{$C', q', \XSP', h\cdot ⟨q, s, p⟩$}
    \Else
        \State {\bf non-det.~branch} $a \in \mathcal A$ and $q' \in \{0,\ldots,N-1\}$:
        \State $C' \gets C \cup \{q \xrightarrow{\Omega(s) / a} q' \}$
		\State $\XSP' \gets$ \Call{NextStates$_{\Prob}$}{$s,a$}
        \State {\bf return} \Call{AND-step$_{\Prob, N}$}{$C', q', \XSP', h\cdot ⟨q, s, p⟩$}
    \EndIf
\EndFunction
\algstore{foo}
\end{algorithmic}

\caption{The \textsc{Pandor} algorithm, which synthesizes finite state controllers with looping histories.\label{alg:pandor}}
\end{algorithm}

%% file: algo2.tex
\begin{algorithm}
\begin{algorithmic}[1]
\algrestore{foo}
\Function{\interfuncstrut{}CumulateAlpha}{$h^{(0:n)}, \bm{\alpha}$}\label{algline:cumulate-alpha}
    \ForAll{$x \in \{\Xgoal, \Xfail, \Xnoter\}$}
	    \State $α^{(n)}_x \gets α^{(n)}_x + p^{(n)} α^{(n+1)}_x / (1 - α^{(n,n)}_{\Xloop})$
	\EndFor
	\For{$k \gets 0 \ldots n - 1$}\label{algline:calc-lambda-loop-for-begin}
            \State $α^{(k,n-1)}_{\Xloop} \gets
                    α^{(k,n-1)}_{\Xloop} + p^{(n)} α_{\Xloop}^{(k,n)} / (1 - α_{\Xloop}^{(n,n)})$
            \State $α_{\Xloop}^{(k,n)} \gets 0$
    \EndFor\label{algline:calc-lambda-loop-for-end}
	\State $α^{(n,n)}_{\Xloop} \gets 0$
    \State {\bf return} $\bm{\alpha}$
\EndFunction
\Function{NextStates$_{\Prob}$}{$s,a$}
	\If{$a=\Xstop$ and $s \in \Goals$}
	    \State \textbf{return} $\{⟨\Xswin, 1.0⟩\}$
	\ElsIf{$a = \Xstop$ and $s \notin \Goals$}
	    \State \textbf{return} $\{⟨\Xsfail, 1.0⟩\}$
	\Else
		\State \textbf{return} $ \{⟨s',p'⟩ \given Δ(s' \given s,a) = p' > 0\}$
	\EndIf
\EndFunction

\Function{\interfuncstrut{}CalcLambda}{$h^{(0:n)},\bm{\alpha}$}
	\State $λ_{x} \gets α^{(n+1)}_{x},$\quad {\bf for} $x \in \{\Xgoal, \Xfail, \Xnoter\}$
	\State $λ_{\Xloop}^{(0:n)} \gets [0, 0, \ldots, 0]$

	\For{$k \gets n \,\ldots\, 0$}
	    \State $λ_{\Xloop}^{(k)} \gets 0$
	    \For{$m \gets n \,\ldots\, k+1$}
	        \State $λ_{\Xloop}^{(k)} \gets λ_{\Xloop}^{(k)} + α_{\Xloop}^{(k,m)} / (1 - λ_{\Xloop}^{(m)})$
	        \State $λ_{\Xloop}^{(k)} \gets p^{(m)} \cdot λ_{\Xloop}^{(k)}$
	    \EndFor
	    \State $λ_{\Xloop}^{(k)} \gets λ_{\Xloop}^{(k)} + α_{\Xloop}^{(k,k)}$
	    \If{$λ_{\Xloop}^{(k)} + λ_{\Xnoter} \approx 1$}\label{algline:loop-noter-begin}
		    \State $λ_{\Xloop}^{(k)} \gets 0$
		    \State $α_{\Xloop}^{(k:, k:)} \gets 0$
		    \State $α_{\Xnoter}^{(k)} \gets α_{\Xnoter}^{(k)} + p^{(k)}$\label{algline:loop-noter-end}
	        \ForAll{$x \in \{\Xgoal, \Xfail, \Xnoter\}$}
			    \State $λ_{x} \gets α_x^{(k)}$
		    \EndFor
	    \Else
    		\ForAll{$x \in \{\Xgoal, \Xfail, \Xnoter\}$}
    			\State $λ_{x} \gets α_x^{(k)} + p^{(k)}\,λ_x/(1 - λ_{\Xloop}^{(k)})$
    		\EndFor
		\EndIf
	\EndFor
	\State \textbf{return} $⟨λ_{\Xgoal}, λ_{\Xfail}, λ_{\Xnoter}⟩$
\EndFunction
\end{algorithmic}
\caption{Helper functions used by {\sc Pandor}.}\label{alg:pandor-helper}
\end{algorithm}

%% file: tikz-loops-in-loops.tex

\fontsize{8}{8}\selectfont

\def\xa{10mm}
\tikzset{level distance=\xa}
\tikzset{sibling distance=6mm}

\tikzset{fake line/.style={dashed}}

\tikzset{edge from parent/.style={-To,draw}}

\node[and] {} [grow'=right,label position=right]
child { node[or] (0) {}
    child { node[and] {}
        child { node[or] (1-b) {}
            child { node[and] {}
                child { node[or] (2-c) {}
                    child [level distance=6mm,black] { node {$\ldots$} }
                    edge from parent[black!70] node[above,yshift=-1mm,black] {$p_c^{(2)}$}
                }
                child { node[or,label=10:$\checkmark$] (2-b) {}
                    edge from parent[black!70] node[black] {$p_b^{(2)}$}
                }
                child[missing]
            }
            edge from parent node[above] {$p_b^{(1)}$}
        }
        child { node[or] (1-a) {}
            child {  node[and] {}
                child { node[or] (2-a) {}
                    child { node[and] {}
                        child { node[or] (3-c) {}
                            child [level distance=6mm,black] { node {$\ldots$} }
                            edge from parent[black!70] node[pos=0.7,yshift=1mm,black] {$p_c^{(3)}$}
                        }
                        child { node[or,double distance=1.5pt,label=$\checkmark$] (3-b) {}
                            edge from parent[black!70] node[pos=0.7,black] {$p_b^{(3)}$}
                        }
                        child { node[or] (3-a) {}
                            child [level distance=6mm,black] { node {$\ldots$} }
                            edge from parent[black!70] node[pos=0.7,black] {$p_a^{(3)}$}
                        }
                    }
                    edge from parent node[below] {$p_a^{(2)}$}
                }
            }
            edge from parent node[below] {$p_a^{(1)}$}
        }
        child[missing]
    }
    edge from parent node[below] {$p^{(0)}$}
};

\draw[->] (3-a) edge[bend left,fake line] (1-a);
\draw[->] (2-c) edge[bend right,fake line] (0);
\draw[->] (3-c) edge[bend right,fake line] (2-a);

%% file: tikz-probhall-a-one-controller.tex
\footnotesize



\node[state, initial] at (0mm,0mm) (q0) {$q_0$};
\node[state] at (45mm,0mm) (q1) {$q_1$};

\draw[->] (q0) edge[loop above] node[align=center,xshift=1mm] {$A :\ \rightarrow$\\$- :\  \rightarrow$} (q0);
\draw[->] (q0) edge[bend left=8] node[above] {$B :\ \leftarrow$} (q1);
\draw[->] (q1) edge[loop above] node {$- :\ \leftarrow$} (q1);
\draw[->] (q1) edge[bend left=8] node[below,align=center] {$A : \Xstop$\\ $B: \ \leftarrow$} (q0);